\documentclass{article}
\usepackage{ijcai25}

\usepackage{times}
\usepackage{soul}
\usepackage{url}
\usepackage[hidelinks]{hyperref}
\usepackage[utf8]{inputenc}
\usepackage[small]{caption}
\usepackage{graphicx}
\usepackage{amsmath}
\usepackage{amsthm}
\usepackage{booktabs}
\usepackage{algorithm}
\usepackage{algorithmic}
\usepackage[switch]{lineno}

\usepackage{amsmath, amssymb, amsfonts}
\usepackage{textcomp}
\usepackage{xcolor}
\usepackage{multirow}
\usepackage{subcaption}
\usepackage{arydshln}

\newtheorem{theorem}{Theorem}

\newcommand{\rdim}[1]{{\mathbb{R}}^{#1}}
\newcommand{\set}[3]{\{#1\}_{#2}^{#3}}
\newcommand{\seto}[1]{\{#1\}}
\newcommand{\mdet}[1]{\left\lvert#1\right\rvert}
\newcommand{\diag}[1]{{\rm diag}(#1)}
\newcommand{\Y}[1]{{\cal Y}_{#1}}

\DeclareMathOperator{\Sim}{{\mathbf{S}}}
\DeclareMathOperator{\Lmat}{{\mathbf{L}}}
\DeclareMathOperator{\Amat}{{\mathbf{A}}}
\DeclareMathOperator{\Rmat}{{\mathbf{R}}}
\DeclareMathOperator{\E}{{\mathbf{E}}}
\DeclareMathOperator{\expm}{{{\mathbf{expm}}}}
\DeclareMathOperator{\logm}{{{\mathbf{logm}}}}
\DeclareMathOperator{\tr}{{\rm{tr}}}

\DeclareMathOperator{\M}{{\mathbf{M}}}

\newcommand{\first}[1]{{\bf{{#1}}}}
\newcommand{\second}[1]{{\underline{#1}}}

\newcommand{\fourth}[1]{{#1}}

\newcounter{definition}
\newcommand{\definition}[1]{\refstepcounter{definition}{\noindent  \bf Definition \arabic{definition} (#1).~}}

\title{MS-DPPs: Multi-Source Determinantal Point Processes for Contextual Diversity Refinement of Composite Attributes in Text to Image Retrieval
}

\author{
Naoya Sogi$^1$
\and
Takashi Shibata$^1$\and
Makoto Terao$^{1}$\and
Masanori Suganuma$^2$\and
Takayuki Okatani$^{2,3}$
\\
\affiliations
$^1$Visual Intelligence Research Laboratories, NEC Corporation\\
$^2$Graduate School of Information Sciences, Tohoku University\\
$^3$RIKEN Center for AIP\\
\emails
naoya-sogi@nec.com, t.shibata@ieee.org, m-terao@nec.com,
\\
\{suganuma, okatani\}@vision.is.tohoku.ac.jp
}

\begin{document}
    \maketitle  

    \begin{abstract}
        Result diversification (RD) is a crucial technique in Text-to-Image Retrieval for enhancing the efficiency of a practical application.
        Conventional methods focus solely on increasing the diversity metric of image appearances.
        However, the diversity metric and its desired value vary depending on the application, which limits the applications of RD.
        This paper proposes a novel task called CDR-CA (Contextual Diversity Refinement of Composite Attributes).
        CDR-CA aims to refine the diversities of multiple attributes, according to the application's context.
        To address this task, we propose Multi-Source DPPs, a simple yet strong baseline that extends the Determinantal Point Process (DPP) to multi-sources.
        We model MS-DPP as a single DPP model with a unified similarity matrix based on a manifold representation.
        We also introduce Tangent Normalization to reflect contexts.
        Extensive experiments demonstrate the effectiveness of the proposed method.
        Our code is publicly available at {\url{https://github.com/NEC-N-SOGI/msdpp}}.
    \end{abstract}

    \section{Introduction}
    Text-to-Image Retrieval (T2IR) aims to retrieve images related to input text from a set of images~\cite{DeVise,VSEImprovingVisualSemantic,StackedCrossAttentiona,ImagetextRetrievalSurvey,wei2025uniir}.
    Recent advances in visual language models have become a new paradigm in the field of computer vision tasks, including T2IR~\cite{chen2020uniter,clip,albef,blip,beit,X22VLMAllinOnePreTrained}.
    For example, BLIP-2 and its variants can perform searches with significantly higher accuracy than conventional methods~\cite{BLIP2BootstrappingLanguageImagea,sogi2025object}.

    Result diversification (RD) is a fundamental task that aims to output relevant but diverse items so that users can comprehensively grasp the information within a gallery dataset~\cite{ImprovingWebSearch,ResultDiversificationSocial,BenchmarkingImageRetrieval}.
    RD methods are demanded for various applications~\cite{pmlr-v32-affandi14,BenchmarkingImageRetrieval,ResultDiversificationSearch} to enhance the efficiency of a practical T2IR application.
    The determinantal point process (DPP) is one of the standard approaches that retrieves the diverse samples based on the relevance score (i.e., similarity between the text query and images) and the appearance similarity between images~\cite{DeterminantalPointProcesses}.
    Conventional RD methods focus solely on increasing the diversity of image appearances.

    Images typically contain various composite attributes such as shooting time and region, and the target attributes for RD should differ for each application.
    Furthermore, whether the diversity of each attribute is increased or decreased depends on the application.
    For example, if the goal is to retrieve images related to disaster damage and grasp the overall damages from the retrieved images~\cite{I1M,sogi2024disaster}, it is necessary to ensure high diversity in terms of both time and area.
    Conversely, if the goal is to analyze disaster damages in a local area deeply, it is effective to search for images taken in a concentrated location with diverse appearances.
    There is a strong demand for a versatile diversity-aware framework that can incorporate application-adaptive contexts in a plug-and-play manner while inheriting the strong capability of visual language models.

    \begin{figure*}[tb]
        \centerline{\includegraphics[width=0.87\textwidth]{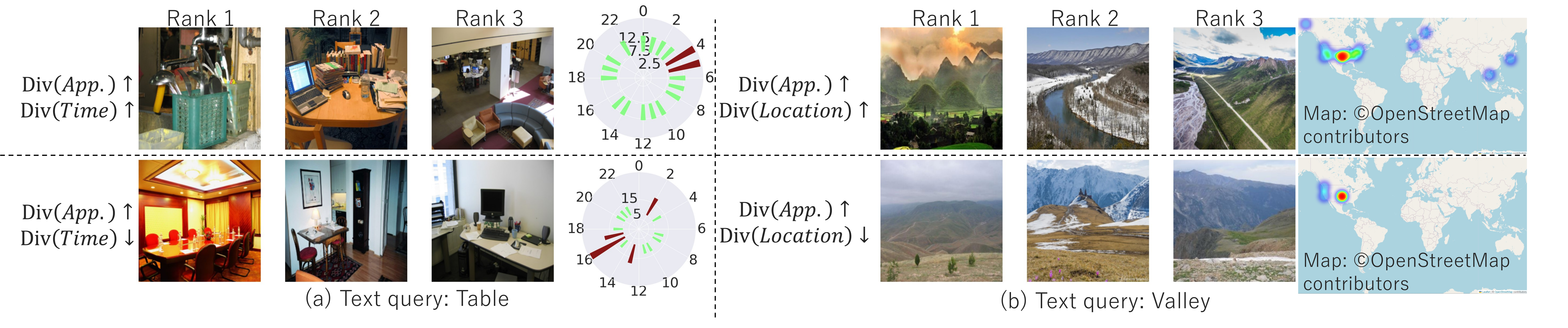}}
        \caption{Examples by MS-DPP for four CDR-CA tasks. 
        The left figures show shooting time refinement results, including the top three images and polar histograms of shooting times in the top 20 images.
        The upward bar of each histogram is at 0:00, and time moves clockwise, rounding the clock in 24 hours.
        The right figures show location refinement results, including heatmaps of the locations in the top 20 images.}
        \label{fig:task_result}
    \end{figure*}

    This paper proposes a generalized task called CDR-CA (Contextual Diversity Refinement of Composite Attributes), which aims to simultaneously refine the diversities of multiple attributes, such as image appearance and shooting time.
    To solve this task, we propose MS-DPPs (Multi-Source DPPs), a simple yet strong baseline that extends DPP to multiple sources.
    The key to MS-DPPs is that they can consider multiple attributes of each image in an application-adaptive and plug-and-play manner.
    We present that MS-DPP can be regarded as a DPP model with a unified similarity matrix, which is an interpolation of attributes' similarity matrices on the tangent space of the Symmetric Positive Definite (SPD) manifold.
    This formulation enables us to employ established optimization techniques from the standard DPPs.
    Consequently, MS-DPPs can refine the diversity of multiple attributes by reflecting each context while inheriting the strengths of existing retrieval frameworks, as shown in Fig.~\ref{fig:task_result}.
    For instance, in Fig.~\ref{fig:task_result}, CDR-CA requires us to provide a list of images with diverse appearance and shooting time.

    In addition to retrieval accuracy and diversity measures, it is important to faithfully reflect attribute weights (or user preferences) in a user-interactive system.
    For instance, if a user increases the importance of an attribute, the system should further improve its diversity.
    To address this issue, we introduce Tangent Normalization (TN) to MS-DPPs. 
    TN is a normalization process before computing an interpolation of attributes' similarity matrices.
    This process suppresses the effect of norms of the similarity matrices in the tangent space so that attribute weights can be effectively reflected.

    Extensive experiments demonstrate that our method outperforms the state-of-the-art methods in terms of mean of retrieval accuracy and diversity.

    The contributions of this paper are fourfold: 1) We propose a novel problem setting called the CDR-CA task, which aims to refine the diversity of multiple attributes simultaneously. 2) We propose MS-DPPs, a simple yet strong baseline that extends the DPP manner to multiple sources. 3) We propose Tangent Normalization (TN) to reflect application contexts in MS-DPPs. 4) We confirmed the effectiveness of the proposed method using three publicly available datasets.

    \section{Related Works}

    \subsection{Text to Image Retrieval}
    Text-to-image retrieval (T2IR) is a pivotal task in the field of vision and language research~\cite{DeVise,VSEImprovingVisualSemantic,StackedCrossAttentiona,ComparativeAnalysisCrossmodal,ImagetextRetrievalSurvey}. 
    T2IR aims to retrieve images from a vast database so that the retrieved images are relevant to a given text query. 
    This task is typically achieved by predicting a relevance score for a pair of each image and the query text.
    Two primary approaches dominate relevance prediction: Image-Text Contrastive (ITC)~\cite{clip,yu2022coca,chen2023internvl} and Image-Text Matching (ITM) methods~\cite{VSEImprovingVisualSemantic,StackedCrossAttentiona}.

    ITC determines relevance using a similarity measure between a query feature and an image feature, commonly using cosine similarity.
    This approach is recognized as fast and effective for T2IR. In contrast, ITM treats relevance prediction as a classification problem, assessing whether an image and a query form a matching pair. While ITM tends to be more accurate than ITC, ITM is also more computationally intensive.
    A hybrid approach is often employed~\cite{BLIP2BootstrappingLanguageImagea} to leverage the advantages of the methods: ITC first retrieves the top-$K$ images, which are then re-ranked by ITM.

    Our MS-DPP refines the diversity in retrieval results by re-ranking the original results obtained by a T2IR method.
    As the re-ranking is a post-processing step, our method can be combined with any T2IR method.
    
    \subsection{Result Diversification}
    Result Diversification (RD) is a fundamental task in information retrieval~\cite{mmr,ImprovingWebSearch,QueryResultDiversification,ResultDiversificationSocial,FastGreedyMAP,BenchmarkingImageRetrieval,ResultDiversificationSearch}.
    This task aims to output relevant but diverse items so that users can comprehensively understand the information in a dataset.

    \subsubsection{Result Diversification in Image Retrieval}
    In the context of image retrieval, RD is a task that increases the diversity of image appearance in the top-$K$ images while maintaining highly relevant images to a text query~\cite{VisualDiversificationImage,HierarchicalClusteringPseudorelevance,ResultDiversificationSocial,HypergraphBasedRerankingModel,PseudorelevanceFeedbackDiversification,FCAbasedKnowledgeRepresentation,BenchmarkingImageRetrieval,OverviewClusterbasedImage,KeywordBasedDiverseImage}.
    This task is important to provide a comprehensive view of the relevant images in a gallery database.
    There are three primary approaches to RD: 1) Maximal Marginal Relevance (MMR)-based methods~\cite{mmr}, 2) Clustering-based methods~\cite{PseudorelevanceFeedbackDiversification}, and 3) Determinantal Point Process (DPP)-based methods~\cite{1360574095929215232,KDPPsFixedsizeDeterminantal}.
    These methods are post-processing techniques, i.e., they re-rank images using relevance scores and image features.
    
    MMR-based methods sequentially select an image that maximizes the sum of relevance to a query and a diversity measure between a candidate and previously selected images. This method is widely used in RD, as it is simple and effective.
    Clustering-based methods sequentially select an image from each cluster while referring to relevances. This approach significantly increases diversity, although it decreases retrieval accuracy. To improve the trade-off, some extensions are proposed, e.g., using pseudo-labeling.
    DPP-based methods use a probabilistic model on subsets of all images. A probabilistic model is defined by a matrix determinant and assigns a high probability to a relevant and diverse subset. 
    As DPPs are effective, they are widely used in RD and actively studied, such as to improve scalability~\cite{FastGreedyMAP} and further understanding of their mechanism~\cite{DeterminantalPointProcesses,pmlr-v89-mariet19b}.
   
    We propose a novel task, Contextual Diversity Refinement of Composite Attributes (CDR-CA), to open up novel applications of RD in T2IR.
    CDR-CA is an extension of RD that refines the diversity of multiple attributes and has mixed diversity directions, i.e., increasing some attributes and decreasing others.
    We then propose Multi-Source DPP (MS-DPP) tailored method of DPPs to CDR-CA.
   
    \section{Preliminary}
    \subsection{Result Diversification}
    Let us begin with an introduction to the standard result diversification
    task in image retrieval. Let $\set{{\mathbf{f}}_{i}\in \rdim{d_I}}{i=1}{N_I}$
    be image features of gallery images $\set{I_{i}}{i=1}{N_I}$ and $q$ be a query
    text. 
    Text-to-image retrieval methods output a top-$K$ ranked list $\Y{K}$ of images by calculating a relevance score $r_{i}$ between the query text and each image $I_{i}$. The result diversification task is to increase the diversity of image appearance in the top-$K$ ranked list, while maintaining high-relevance images.
    This concept can be written as follows:
    \begin{align}
        {\rm argmax}_{\Y{K}}(1-\theta)\sum_{i\in\Y{K}}r_{i}+\theta{\rm Div}(\Y{K}),\label{eq:task}
    \end{align}
    where $\theta$ is a trade-off parameter between retrieval accuracy and diversity, and ${\rm Div}(\Y{K})$ is a diversity measurement, such as the negative maximum pairwise similarity of images in $\Y{K}$~\cite{ResultDiversificationSearch}.

    \subsection{Result Diversification by DPPs}
    In the following, we describe a DPP-based method~\cite{KDPPsFixedsizeDeterminantal}, which is the basis of our method.
    \subsubsection{Notations}
    We introduce notations for DPPs.
    Let $N_I$ be the number of images in a gallery, $\set{{\mathbf{f}}_{i}\in \rdim{d_I}}{i=1}{N_I}$ be image features.
    A retrieval model outputs relevance scores $\set{r_{i}}{i=1}{N_I}$ between a query text and images.
    A retrieval result, i.e., top-$K$ ranked list, is represented by an ordered index set $\Y{K}$.

    DPPs use three matrices: $\Rmat$, $\Sim$, and $\Lmat$.
    $\Rmat\in\mathbb{R}^{N_I\times N_I}=\diag{r_1,r_2,...,r_{N_I}}$.
    $\Sim\in\mathbb{R}^{N_I\times N_I}$ is a similarity matrix of image features.
    $\Lmat\in\mathbb{R}^{N_I\times N_I}=\Rmat\Sim\Rmat$ defines a DPP.
    Let $\E$ be an identity matrix.
    A matrix with subscript $\Y{g}$ is a submatrix indexed by $\Y{g}$.
    For instance, $\Sim_{\Y{g}}\in\mathbb{R}^{g\times g}$ is a submatrix of $\Sim$ indexed by $\Y{g}$: $(\Sim_{\Y{g}})_{(i,j)}=(\Sim)_{x,y}, x=\Y{g}(i), y=\Y{g}(j)$, where $\Y{g}(i)$ is the $i$th element of $\Y{g}$.

    Let $\log$ be the natural logarithm, $\logm$ and $\expm$ be the matrix logarithm and matrix exponential, and $\mdet{\cdot}$ and $\tr{\cdot}$ be the matrix determinant and trace, respectively.

    \subsubsection{Determinantal Point Process (DPP)}
    DPP increases the diversity of image appearance in the top-$K$ ranked list, while maintaining highly relevant images by using the following probabilistic model:
    \begin{align}
        p(\Y{g})=\frac{\mdet{\Lmat_{\Y{g}}}}{\mdet{\Lmat+\E}}
        \propto\mdet{\Lmat_{\Y{g}}}={\mdet{\Rmat_{\Y{g}}}\mdet{\Sim_{\Y{g}}}\mdet{\Rmat_{\Y{g}}}}.\label{eq:dpps}
    \end{align}
    Large determinants of $\Rmat_{\Y{g}}$ and $\Sim_{\Y{g}}$, i.e., large probability, indicate that the images in $\Y{g}$ are relevant and diverse, respectively.
    Please refer to the paper~\cite{DeterminantalPointProcesses} for the detailed mechanism.
    We can obtain a diversified image set as $\Y{g}^*={\rm argmax}_{\Y{g}} \mdet{\Rmat_{\Y{g}}\Sim_{\Y{g}}\Rmat_{\Y{g}}}$.
    This optimization problem is generally NP-hard, but we can obtain $\Y{g}^*$ efficiently by sequentially selecting images that maximize the determinant~\cite{KDPPsFixedsizeDeterminantal,FastGreedyMAP}.

    \begin{figure*}[t]
        \centerline{\includegraphics[width=0.85\textwidth]{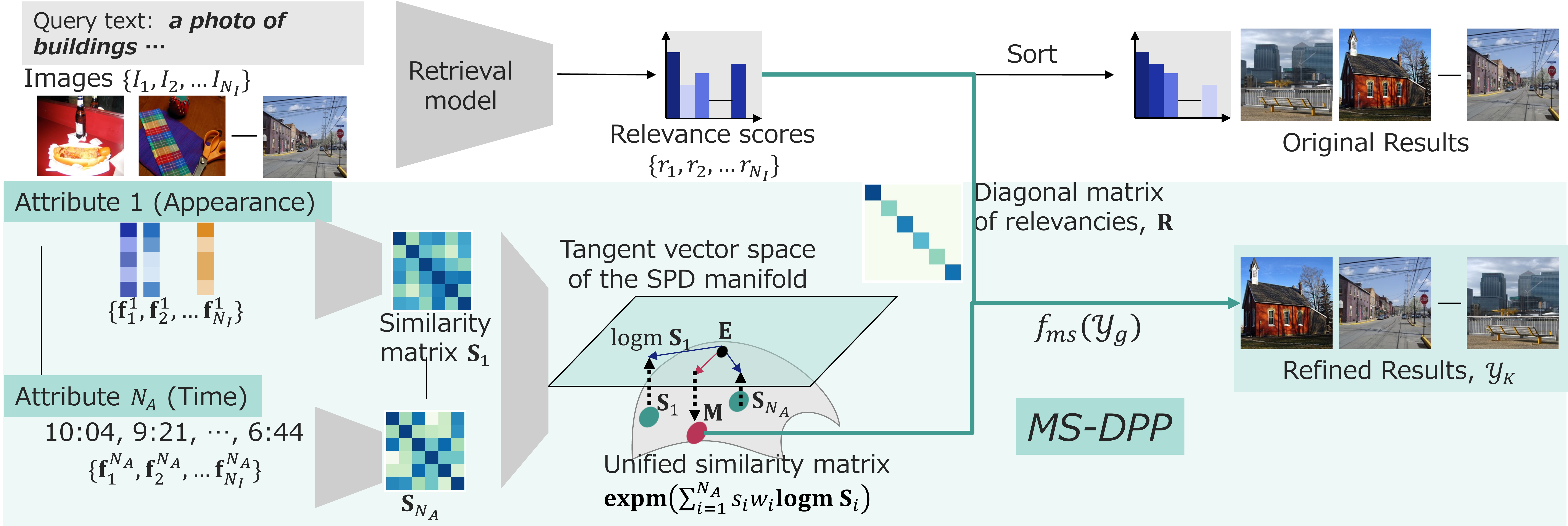}}
        \caption{Overview of MS-DPP. We first calculate the similarity matrices of each attribute and then unify them into a single similarity matrix through operations, matrix logarithm and exponential, on the SPD manifold. Given the unified similarity matrix and the diagonal matrix of relevances, MS-DPP re-ranks images to refine the diversity of multiple attributes.}
        \label{fig:overview}
    \end{figure*}

    \section{Problem Formulation: Contextual Diversity Refinement of Composite Attributes}

    Our Contextual Diversity Refinement of Composite Attributes (CDR-CA) task is to refine the diversity of multiple attributes, such as image appearance and shooting time, while reflecting relevance scores, as shown in Fig.~\ref{fig:task_result}.
    For example, CDR-CA requires us to provide a list of images that 1) have diverse appearances and shooting times or 2) have diverse appearances while taken at a concentrated time.
    The concept of CDR-CA can be written as follows:
    \begin{align}
        {\rm argmax}_{\Y{K}}(1-\theta)\sum_{i\in\Y{K}}r_{i}+\theta\sum_{j=1}^{N_A}s_{j}w_{j}{\rm Div}_{j}(\Y{K}),\label{eq:task_a2da}
    \end{align}
    where $N_{A}$ is the number of attributes, $s_{j}\in\seto{-1, +1}$ is an indicator parameter to determine diversity of $j$th attribute is increased (+1) or decreased (-1), $w_{j}$ is the weight of $j$th attribute, and ${\rm Div}_{j}(\Y{K})$ is the diversity measurement of $j$th attribute of images in $\Y{K}$. 
    For simplicity, we add image appearance as the first attribute.
    Table~\ref{tab:diff} summarizes the differences between CDR-CA and the standard result diversification.

    \begin{table}
        \centering
        \scalebox{0.75}{
        \begin{tabular}{l|ccc}
            \toprule
            & $N_A$ & Diversity Direction $s$ & Preference $w$\\
            \midrule
            RD & 1 & $s=+1$ & $w=1$ \\
            CDR-CA & $>1$ & $s_i\in\{-1, +1\}$ & $w_i\in[0,1]$ \\
            \bottomrule
        \end{tabular}
        }
        \caption{Differences between the Result Diversification (RD) and the proposed CDR-CA tasks.}
        \label{tab:diff}
    \end{table}

    
    \section{Proposed Method}
    This section presents the proposed method, Multi-Source Determinantal Point Process (MS-DPP), a novel extension of the traditional Determinantal Point Processes (DPPs) tailored to address CDR-CA.
    Figure~\ref{fig:overview} illustrates the overview of the proposed method.
    We first describe the basic formulation of MS-DPP using multiple DPPs to handle multiple attributes.
    We then demonstrate that the basic formulation can be seamlessly transformed into a single DPP model by unifying similarity matrices of multiple attributes through operations on the SPD manifold.
    This unification enables us to leverage established optimization techniques for DPPs.
    Finally, we introduce Tangent Normalization (TN), a novel technique that aligns attribute diversities with user-defined weights.

    \subsection{Multi-Source Determinantal Point Processes}
    \subsubsection{Basic Idea}
    MS-DPP is an extension of the DPP framework to address a CDR-CA task.
    The key concept involves applying multiple DPPs across different attributes to enhance their diversity simultaneously.
    Let us first consider a simple CDR-CA task; all attributes hold equal importance ($w_i=1$), and the diversity of each attribute is increased ($s_i=1$).
    The basic model is defined as follows:
    \begin{align}
        \prod_{i=1}^{N_A}p_i(\Y{g})
        \propto f_{base}(\Y{g})
        =\prod_{i=1}^{N_A}\mdet{\Rmat_{\Y{g}}}\mdet{\Sim_{i,\Y{g}}}\mdet{\Rmat_{\Y{g}}},
        \label{eq:msdpp-base}
    \end{align}
    where $p_i$ is a DPP model for the $i$th attribute defined in \eqref{eq:dpps}, $\Sim_i$ and $\Rmat$ are a similarity matrix of the $i$th attribute and a diagonal matrix of relevances, respectively.
    Maximization of this composite model $f_{base}$ refines the diversity of multiple attributes within the top-$K$ ranked list.

    To efficiently derive the optimal list $\Y{K}^*$, we reformulate the composite DPP model $f_{base}$ into a single DPP model by unifying the similarity matrices $\seto{\Sim_i}$ into one SPD matrix $\M$.
    This unification allows us to employ optimization techniques~\cite{KDPPsFixedsizeDeterminantal,FastGreedyMAP} from the standard DPP framework.

    \subsubsection{Unified Representation for MS-DPP}
    In the following, we first reformulate the composite DPP model into a single DPP model with a unified similarity matrix.
    We then generalize the unified representation to accommodate various cases of CDR-CA tasks, i.e., mixed diversity directions ${s_i\in\{-1, +1\}}$ and flexible weights $w_i\in\mathbb{R}$.

    \begin{theorem}
        {\rm\textbf{Unified Representation:}} The composite DPP model ($f_{base}$) in \eqref{eq:msdpp-base} can be expressed as a single DPP model with a unified similarity matrix $\M$ as follows:
    \begin{align}
        f_{base}(\Y{g})
        \propto f_{unify}(\Y{g})
        =\mdet{\Rmat_{\Y{g}}}\mdet{\M}\mdet{\Rmat_{\Y{g}}},\label{eq:msddp1}
    \end{align}
    where $\M=\expm{\sum_{i=1}^{N_A}{\logm \Sim_{i,\Y{g}}}}$.
    \end{theorem}
    {\it Proof:} Please refer to Sec. 2 in the supplementary material.
    
    The essence of the unified representation lies in transforming the product of multiple matrix determinants into a single determinant by utilizing the matrix logarithm and exponential.
    Consequently, we can leverage existing optimization techniques from the standard DPP framework to efficiently derive the optimal result $\Y{K}^*$ of $f_{unify}$.
    
    \definition{MS-DPP}
    In the above, we consider a simple CDR-CA task where all attributes have equal importance $w_i=1$ and the diversity of each attribute is increased $s_i=1$.
    To tackle various types of CDR-CA tasks, we extend the unified representation as follows: 
    \begin{align}
        f_{ms}({\Y{g}})=\mdet{e^{ \alpha \Rmat_{\Y{g}}}}\mdet{\expm{\sum_{i=1}^{N_A}{s_iw_i\logm\Sim_{i,\Y{g}}}}}\mdet{e^{ \alpha \Rmat_{\Y{g}}}},\label{eq:msdppw} 
    \end{align}
    where $s_i\in\{-1, +1\}$ indicates the diversity refinement direction of the $i$th attribute, $w_i$ is the weight or user's preference of the $i$th attribute, and $e^{ \alpha \Rmat_{\Y{g}}}$ is elementwise exponential of $\alpha\Rmat_{\Y{g}}$.
    $\alpha=\frac{\theta}{2(1-\theta)}, \theta\in(0,1)$ is a trade-off parameter~\cite{FastGreedyMAP} between retrieval accuracy and diversity refinement.
    We refer to $f_{ms}$ as an {\bf MS-DPP} model.
    Algorithm~\ref{alg:msdpps} shows the k-DPP~\cite{KDPPsFixedsizeDeterminantal} based optimization of an MS-DPP model.
    In the experiments, we use the computationally efficient extension~\cite{FastGreedyMAP} of this algorithm.

    \begin{algorithm}[tb]
        \caption{\footnotesize MS-DPP for a CDR-CA task}
        \label{alg:msdpps}
        \footnotesize
        \begin{algorithmic}
            \STATE {\bfseries Input:} Similarity matrices $\set{\Sim_{i}\in\mathbb{R}^{N_I\times N_I}}{i=1}{N_A}$, Relevance matrix $\Rmat\in\mathbb{R}^{N_I\times N_I}$, attribute weights $\seto{w_i}$, refinement directions $\seto{s_i}$, and trade-off parameter $\theta$.
            \STATE {\bfseries Output:} A top-$K$ ranked list $\Y{K}$.
            \STATE Initialize $\Y{g}=\emptyset$, ${\cal I}=\{1,2,...,N_I\}$
            \WHILE{$|\Y{g}|<K$}
            \FOR{$j=1$ {\bfseries to} $N_A$}
                \STATE Calculate $f_j=f_{ms}(\Y{g\cup j})$ \eqref{eq:msdppw}.
            \ENDFOR
            \STATE Update $\Y{g}=\Y{g}\cup\seto{j^*}$, ${\cal I}={\cal I}\setminus\seto{j^*}$, where $j^*={\rm argmax}_{j}f_j$.
            \ENDWHILE
        \end{algorithmic}
    \end{algorithm}

        \begin{figure}[t]
            \centering
            \includegraphics[width=\columnwidth]{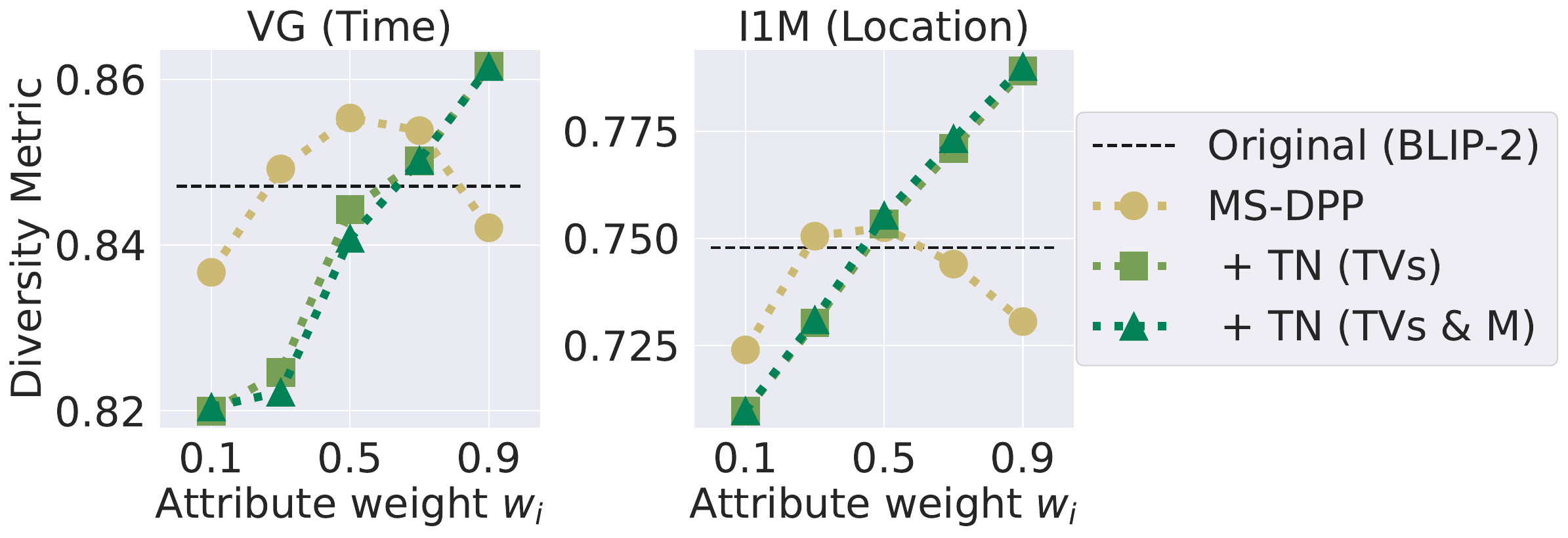}
        \caption{Transition of the diversity of an attribute varying its weight $w_i$ from 0.1 to 0.9 with a step size of 0.2 with $\theta=0.9$. The task is to increase diversity in appearance and an attribute (time or location). The higher diversity metric indicates better performance.
        } 
        \label{fig:ana_tradeoff} 
    \end{figure}

    \subsection{MS-DPP with Tangent Normalization}
    \label{sec:tn}
    \subsubsection{Motivation and Basic Idea} 
    Figure~\ref{fig:ana_tradeoff} shows the transition of the diversity of an attribute varying its weight $w_i$. 
    The details of this experiment will be described in Sec.~\ref{sec:exp}.
    MS-DPP improves the diversity metric, although MS-DPP with a large weight $w_i>0.5$ does not always increase the metric as much as the weight.
    It is important for a user-interactive system to reflect weights faithfully.
    To address this issue, we introduce a novel normalization, Tangent Normalization (TN).
   
    The reason for the above issue can be understood from the fact that the unification $(\sum_{i=1}^{N_A}{s_iw_i\logm\Sim_{i,\Y{g}}})$ of similarity matrices can be regarded as an interpolation of the tangent vectors $\seto{\logm \Sim_{i,\Y{g}}}$ (TVs) on the SPD manifold.
    This interpretation arises because the matrix logarithm $\logm$ maps an SPD matrix to the tangent space, an Euclidean space, with the identity matrix as the origin (please refer to Sec. 1 in the supplementary material).
    This interpretation suggests that the norm of each TV can influence the balance of diversity refinement among attributes, i.e., an attribute with a large norm can be dominant, although a user makes large importance on other attributes.

    \subsubsection{Tangent Normalization}
    To mitigate the influence of the norm of each TV, we normalize TVs before taking the weighted sum.
    Considering that $\Rmat$ is also an SPD matrix and we need to consider the trade-off between retrieval accuracy and diversity refinement, we propose the following normalization:
    \begin{align}
        \Amat'_{i,{\cal Y}_g}=\frac{b}{\|\logm\Sim_{i,\Y{g}}\|_F}\logm\Sim_{i,\Y{g}}, \label{eq:tns}
    \end{align}
    where $\|\cdot\|_F$ is the Frobenius norm, i.e., the norm in the tangent vector space, and $b=\|\logm \Rmat_{\Y{g}}\|_F$ is the norm of the tangent vector of the relevance matrix.
    We call this normalization as {\bf Tangent Normalization (TN)}.
    Figure~\ref{fig:tn_overview} shows the conceptual diagram of TN.

    Additionally, we normalize the mean of normalized vectors (${\M'}= \expm\sum_{i=1}^{N}{s_iw_i\Amat_{i,\Y{g}}'}$) to faithfully reflect the accuracy-diversity trade-off, as follows:
    \begin{align}
        \logm\M''=\frac{b}{\|\logm{\rm\bf \M'}\|_F} \logm{\rm\bf \M'}.\label{eq:tnr}
    \end{align}
    We use $\M'$ or $\M''$ in the MS-DPP model \eqref{eq:msdppw} instead of $\M$ to reflect the user's preferences $\seto{w_i}$ among attributes.
    As shown in Fig.~\ref{fig:ana_tradeoff}, MS-DPP with TN improves the diversity metric as the weight increases.
    This improvement will be quantitatively confirmed in the experiment section using the evaluation index defined in the following subsection.

    \subsection{Preference Reflection Score (PRS)}
    To assess how well a user's preference $w_i$ for an attribute is reflected, we introduce the Preference Reflection Score (PRS).
    Given refinement results $\set{\Y{g,j}}{j=1}{T}$ and the user's preference $\set{w_{i,j}}{j=1}{T}$, ($w_{i,1}\leq w_{i,2}\leq...\leq w_{i,T}$), for the $i$th attribute in the $j$th result, we calculate the PRS as follows:
    \begin{align}
        {\rm PRS} = \sum_{j=2}^{T}{\frac{(\overline{\rm {Div}}_{i}(\Y{g,j})-\overline{\rm {Div}}_{i}(\Y{g,j-1}))}{(w_{i,j}-w_{i,j-1})}},
        \label{eq:prs}
    \end{align}
    where ${\rm \overline{Div}}_{i}(\Y{g,j})$ is the diversity measurement of the $i$th attribute in the top-$g$ ranked list $\Y{g,j}$.
    ${\rm \overline{Div}}_i$ is normalized to $\max_{j}{\rm \overline{Div}}_{i}(\Y{g,j})=1, \min_{j}{\rm \overline{Div}}_{i}(\Y{g,j})=0$.
    PRS increases when the diversity of the $i$th attribute grows in alignment with the rising user's preference. 
    In the subsequent experiments, we generate eleven results by changing the user's preference from 0 to 1 in increments of 0.1 and then calculate the PRS.


    \begin{figure}[t]
        \centerline{\includegraphics[width=0.45\linewidth]{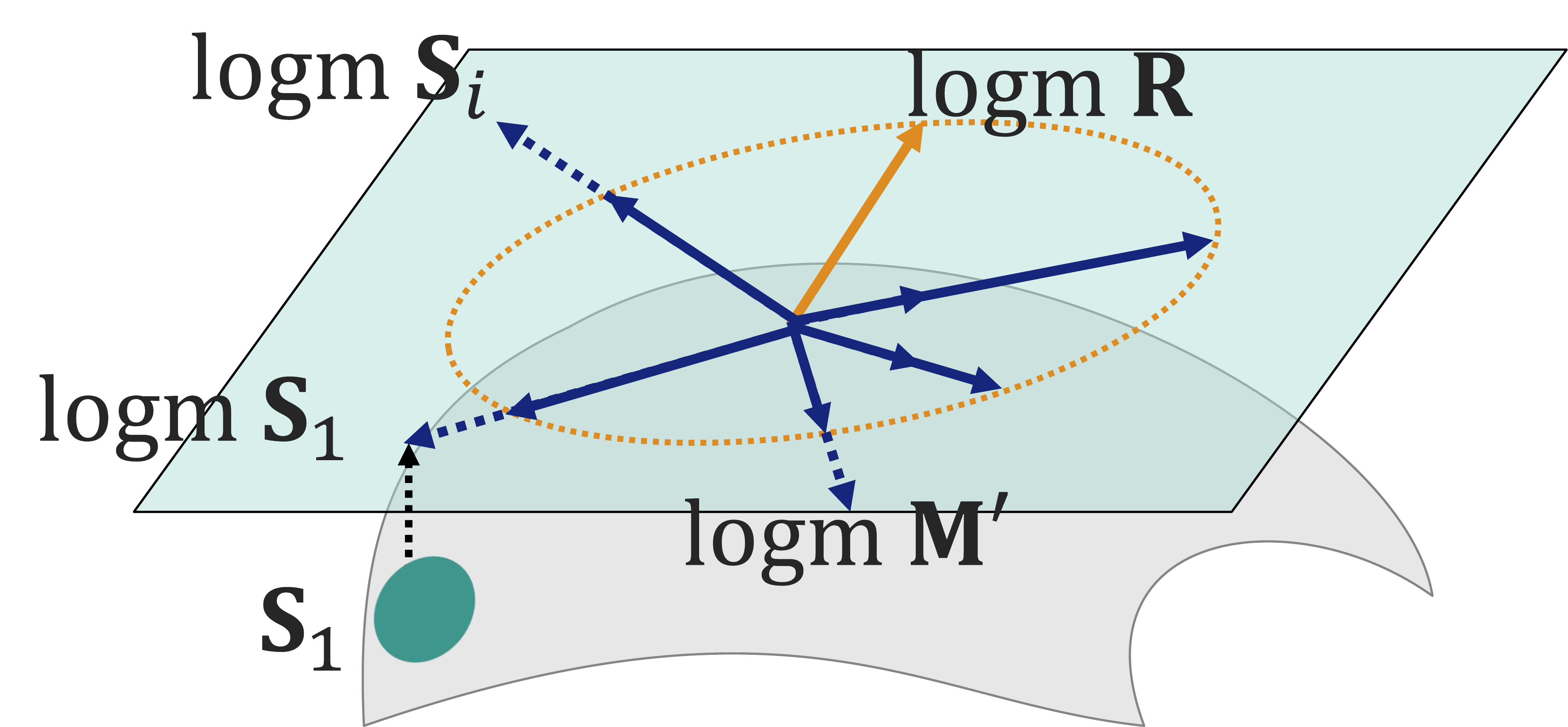}}
        \caption{Conceptual diagram of Tangent Normalization (TN). 
        To suppress the influence of the norm of each tangent vector, TN normalizes tangent vectors of similarity matrices $\seto{\Sim_i}$ and the unified similarity matrix $\M'$ by the norm of the tangent vector of $\Rmat$.}
        \label{fig:tn_overview}
    \end{figure}

    \section{Experiments}

\begin{table*}[t]
    \begin{center}
        \scalebox{0.75}{
        \begin{tabular}{l|ll|rrrr}
\toprule
            &&& {\bf VG (Time)} & {\bf I1M (Location)} & {\bf PP (Time)} & {\bf PP (Location)} \\
\midrule
            &App.&Att.& HM$\uparrow$ \footnotesize{(MAP$\uparrow$,DM$\uparrow$)} & HM$\uparrow$ \footnotesize{(MAP$\uparrow$,DM$\uparrow$)}& HM$\uparrow$\footnotesize{(N@10$\uparrow$,DM$\uparrow$)}& HM$\uparrow$ \footnotesize{(N@10$\uparrow$,DM$\uparrow$)}\\
\midrule
            BLIP-2&-&- & 0.8450 \footnotesize{(0.8259, 0.8651)} & 0.7948 \footnotesize{(0.7917, 0.7979)} & 0.7244 \footnotesize{(0.6214, 0.8683)} & 0.7181 \footnotesize{(0.6214, 0.8503)} \\
\midrule
            Clustering&\checkmark&- & 0.8497 \footnotesize{(0.8156, 0.8868)} & 0.7870 \footnotesize{(0.7710, 0.8038)} & 0.6970 \footnotesize{(0.5781, 0.8775)} & 0.6938 \footnotesize{(0.5781, 0.8675)}\\
            MMR&\checkmark&- & 0.8467 \footnotesize{(0.8222, 0.8726)} & 0.7922 \footnotesize{(0.7848, 0.7998)} & 0.7226 \footnotesize{(0.6176, 0.8704)} & 0.7167 \footnotesize{(0.6176, 0.8535)} \\
            k-DPP&\checkmark&- & 0.8472 \footnotesize{(0.8284, 0.8669)} & 0.7950 \footnotesize{(0.7901, 0.8000)} & 0.7233 \footnotesize{(0.6194, 0.8693)} & 0.7171 \footnotesize{(0.6194, 0.8516)} \\
\midrule
            Clustering&-&\checkmark & \second{0.8557} \footnotesize{(0.8071, 0.9104)} & \second{0.8056} \footnotesize{(0.7837, 0.8288)} & 0.7226 \footnotesize{(0.6086, 0.8891)} & 0.7140 \footnotesize{(0.5967, 0.8888)} \\
            MMR&-&\checkmark & 0.8475 \footnotesize{(0.8282, 0.8676)} & 0.7968 \footnotesize{(0.7879, 0.8058)} & 0.7207 \footnotesize{(0.6108, 0.8787)} & 0.7123 \footnotesize{(0.6025, 0.8710)} \\
            k-DPP&-&\checkmark & 0.8510 \footnotesize{(0.7930, 0.9183)} & 0.7870 \footnotesize{(0.7122, 0.8793)} & \fourth{0.7265} \footnotesize{(0.6205, 0.8763)} & \second{0.7199} \footnotesize{(0.6199, 0.8583)}  \\
\midrule
            Clustering&\checkmark&\checkmark & 0.8422 \footnotesize{(0.7984, 0.8910)} & \fourth{0.7981} \footnotesize{(0.7799, 0.8171)} & 0.7204 \footnotesize{(0.6060, 0.8878)} & 0.7096 \footnotesize{(0.5916, 0.8864)} \\
            MMR&\checkmark&\checkmark & 0.8464 \footnotesize{(0.8246, 0.8694)} & 0.7964 \footnotesize{(0.7887, 0.8044)} & 0.7151 \footnotesize{(0.5999, 0.8850)} & 0.7136 \footnotesize{(0.6051, 0.8695)} \\
            k-DPP&\checkmark&\checkmark & 0.8534 \footnotesize{(0.7979, 0.9171)} & 0.7758 \footnotesize{(0.7042, 0.8636)} & \second{0.7266} \footnotesize{(0.6208, 0.8759)} & \fourth{0.7197} \footnotesize{(0.6200, 0.8575)} \\
\midrule
            MS-DPP &\checkmark&\checkmark & \first{0.8649} \footnotesize{(0.8220, 0.9124)} & \first{0.8081} \footnotesize{(0.7618, 0.8604)} & \first{0.7269} \footnotesize{(0.6183, 0.8818)} & \first{0.7234} \footnotesize{(0.6193, 0.8695)} \\
\bottomrule
        \end{tabular}
       }
    \caption{Comparative results with baselines on diversity increasing tasks. The best and second-best are highlighted in bold and underlined, respectively. The first row specifies the dataset name and target attributes. The second and third columns indicate whether appearance (App.) and attribute (Att.) information are used as inputs for each method.}
    \label{tab:exp_main_inc}
    \end{center}
\end{table*}

\begin{table}[t]
    \begin{center}
        \scalebox{0.75}{
            \begin{tabular}{l|ll|rr}
\toprule
                &&& {\bf VG (Time)} & {\bf I1M (Location)} \\
\midrule
                &App.&Att.& HM$\uparrow$ \footnotesize{(MAP$\uparrow$,DM$\uparrow$)} & HM$\uparrow$ \footnotesize{(MAP$\uparrow$,DM$\uparrow$)} \\
\midrule
            BLIP-2&-&- & 0.4200 \footnotesize{(0.8259, 0.2816)} & 0.5654 \footnotesize{(0.7917, 0.4398)}  \\
            \midrule
            Clustering&\checkmark&- & 0.4169 \footnotesize{(0.6867, 0.2993)} & 0.5564 \footnotesize{(0.7449, 0.4441)} \\
            MMR&\checkmark&- & 0.4578 \footnotesize{(0.8158, 0.3181)} & 0.5689 \footnotesize{(0.8025, 0.4406)} \\
            k-DPP&\checkmark&- & 0.4198 \footnotesize{(0.8284, 0.2811)} & 0.5637 \footnotesize{(0.7901, 0.4382)}  \\
\midrule
            Clustering&-&\checkmark & 0.5182 \footnotesize{(0.6076, 0.4518)} & 0.5380 \footnotesize{(0.7159, 0.4309)} \\
            MMR&-&\checkmark & 0.5304 \footnotesize{(0.7323, 0.4158)} & 0.6289 \footnotesize{(0.7694, 0.5317)} \\
            k-DPP&-&\checkmark & 0.4039 \footnotesize{(0.7335, 0.2787)} & 0.5515 \footnotesize{(0.7397, 0.4397)} \\
\midrule
            Clustering&\checkmark&\checkmark & 0.5137 \footnotesize{(0.6027, 0.4476)} & 0.5532 \footnotesize{(0.7465, 0.4394)} \\
            MMR&\checkmark&\checkmark & \second{0.5311} \footnotesize{(0.7495, 0.4113)} & \second{0.6301} \footnotesize{(0.7782, 0.5294)} \\
            k-DPP&\checkmark&\checkmark & 0.4917 \footnotesize{(0.6567, 0.3930)} & 0.5995 \footnotesize{(0.6741, 0.5397)} \\
\midrule
            MS-DPP &\checkmark&\checkmark & \first{0.5746} \footnotesize{(0.6616, 0.5079)} & \first{0.6393} \footnotesize{(0.7598, 0.5518)} \\
\bottomrule
        \end{tabular}
   }
    \caption{Comparative results on the tasks of increasing diversity in appearance while decreasing diversity in an attribute. 
            The best and second-best are highlighted in bold and underlined, respectively. 
            The first row shows the dataset name and the target attribute. 
            The second and third columns indicate whether appearance (app.) and attribute (att.) information are used as inputs for each method.}
    \label{tab:exp_main_dec}
    \end{center}
\end{table}

    \subsection{Experimental Settings}
    \label{sec:exp}
    \subsubsection{Experimental Tasks} We evaluate the proposed method on four CDR-CA tasks.
    The first two tasks are diversity increasing tasks: 1) increasing diversity in both appearance and shooting time, and 2) increasing diversity in both appearance and shooting location. 
    The latter two tasks involve mixed objectives: 3) increasing diversity in appearance while decreasing diversity in shooting time, and 4) increasing diversity in appearance while decreasing diversity in shooting location.

    \subsubsection{Datasets} We use three datasets, i.e., Visual Genome (VG)~\cite{vg}, Incidents 1M (I1M)~\cite{weber2020eccv,I1M}, and PixelProse (PP)~\cite{pp}, as they have enough amount of images and have necessary attributes.
    We employ VG for tasks involving shooting time, I1M for tasks involving shooting location, and PP for tasks involving shooting time or location.
    We selected images with the necessary attributes from each dataset.
    Attributes were obtained from EXIF information~\footnote{Please refer to Sec. 3.1 in the supplementary material for the detailed dataset descriptions and preprocessing.}.
    The resulting numbers of images are 649, 22,547, and 25,151 for VG, I1M, and PP, respectively.
    For text queries, we use fifteen object names in VG, the disaster and shooting locations types annotated in I1M, and the 1,000 randomly sampled dense captions in PP.
    We use 20\% of images in each dataset as validation set, with the remaining images constituting the test set.

    \subsubsection{Evaluation Metrics}
    
    {\noindent \bf - Overall Metric:} Our primary metric is the harmonic mean (HM) of the diversity metric and the retrieval accuracy.
    
    {\noindent \bf - Diversity Metric:}
    As the diversity metric (DM), we report the harmonic mean of the Vendi scores ($VS_{0.1}$)~\cite{vendi,vendi2} on each attribute.
    $VS_{0.1}$ is a metric that evaluates the diversity of each attribute in the top-$K$ result.
    The higher Vendi score indicates higher diversity.
    Given that the Vendi score can range from 0 to $K$, we normalize the score as $VS_{0.1}/K$ for attributes whose diversity is to be increased and $1-VS_{0.1}/K$ for attributes whose diversity is to be decreased.
    We set $K=20$ across all tasks.
    A higher value of DM indicates better performance due to the normalization.

    {\noindent \bf - Retrieval Metric:}
    For retrieval accuracy, we use two common metrics: 1) Mean Average Precision (MAP) for VG and I1M, and 2) Normalized Cumulative Semantic Score at 10 (N@10)~\cite{Biten_2022_WACV} for PP.
    MAP represents the average value of AP at each query, while N@10 represents the ratio of the retrieved semantic score to the ground truth semantic score.
    MAP is a suitable metric for VG and I1M due to multiple relevant images per query, whereas N@10 is appropriate for PP, which contains only one relevant image per query.
    A higher value of each metric indicates better performance.

    \subsubsection{Implementation Details}
    We used BLIP-2~\cite{BLIP2BootstrappingLanguageImagea}, a widely used retrieval model, to calculate the relevance score between the query text and each image.
    We used the BLIP-2 model finetuned with the MSCOCO dataset~\cite{lin2014microsoft}, and employed the ITC mechanism of BLIP-2. 

    We employed BLIP-2s' image features to calculate an appearance similarity.
    BLIP-2 outputs 32 features per image. We use the average feature ${{\mathbf f}_i}$ for each image.
    We defined the appearance similarity as the inverse of the distance $1/(\|{\mathbf f}_i-{\mathbf f}_j\|_2+1)$.
    For shooting times or locations, similarity was also evaluated using the inverse of the distance between embedded features.
    Shooting time was embedded as a two dimensional vector $(\cos z, \sin z)$, where $z=\frac{t}{24\times60}\pi, t={\rm hour}\times60+{\rm minute}$.
    Shooting location was embedded as a three dimensional vector $(\cos(lat)\cos(lon), \cos(lat)\sin(lon), \sin(lat))$.

    We first select the top-200 images by BLIP-2 and then re-rank them using the proposed method to refine diversity.

    The hyperparameters of the proposed method were tuned by the grid-search algorithm on the validation set in terms of the HM.
    The trade-off parameter $\theta$ between retrieval accuracy and diversity was selected from 0.75 to 0.95 in increments of 0.05.
    The attribute weights $\seto{w_i}$ were selected from 0.1 to 0.9 in increments of 0.2 and normalized to sum to 1.
    Whether to use TN or not was determined by the grid search algorithm.
    
    \subsubsection{Baselines} 
    We benchmark our method against the original results from BLIP-2 and three single-source diversification techniques: Clustering-based method, MMR, and k-DPP.
    For the diversity increasing tasks, we employ the standard MMR and k-DPP algorithms with average appearance features or embedded attribute features.
    For the diversity decreasing tasks, we slightly modify MMR and k-DPP; we multiply the similarity of each attribute by -1.
    As multi-source baselines, we input concatenated features of appearance and attribute features to the Clustering-based method and MMR.
    Attribute weights $w_i$ and task indicator $s_i$ are multiplied to feature vectors before the concatenation.
    As a multi-source DPP baseline, we input the average weighted similarity matrix to k-DPP.

    The hyperparameters of the baselines were also tuned using the same strategy as the proposed method.
    Please refer to Sec. 3.2 in the supplementary material for the other detailed implementation of the baselines.

    \subsection{Comparative Results}

    \subsubsection{Results for Diversity Increasing Tasks} Table~\ref{tab:exp_main_inc} shows the comparative results of MS-DPP and the baselines for the diversity increasing tasks.
    Our MS-DPP outperforms the baselines in HM for all tasks, demonstrating its capability to enhance the diversity of multiple attributes while preserving retrieval accuracy.
    This result supports the effectiveness of the MS-DPPs' formulation, which is based on the product of multiple DPP models and the aggregation of multiple attributes in the tangent space of the SPD manifold.

    The conventional single-source baselines relying on appearance features alone demonstrate comparatively low performance.
    This underscores the insufficiency of considering only image appearance for CDR-CA tasks, thus highlighting the necessity of integrating multiple attributes.

    Furthermore, our method outperforms multi-source baselines, which show only marginal improvements over their single-source counterparts.
    This indicates that merely aggregating multiple attributes is inadequate for CDR-CA tasks, whereas our method effectively aggregates multiple attributes for better performance.

    \begin{figure}
        \centering
        \includegraphics[width=\columnwidth]{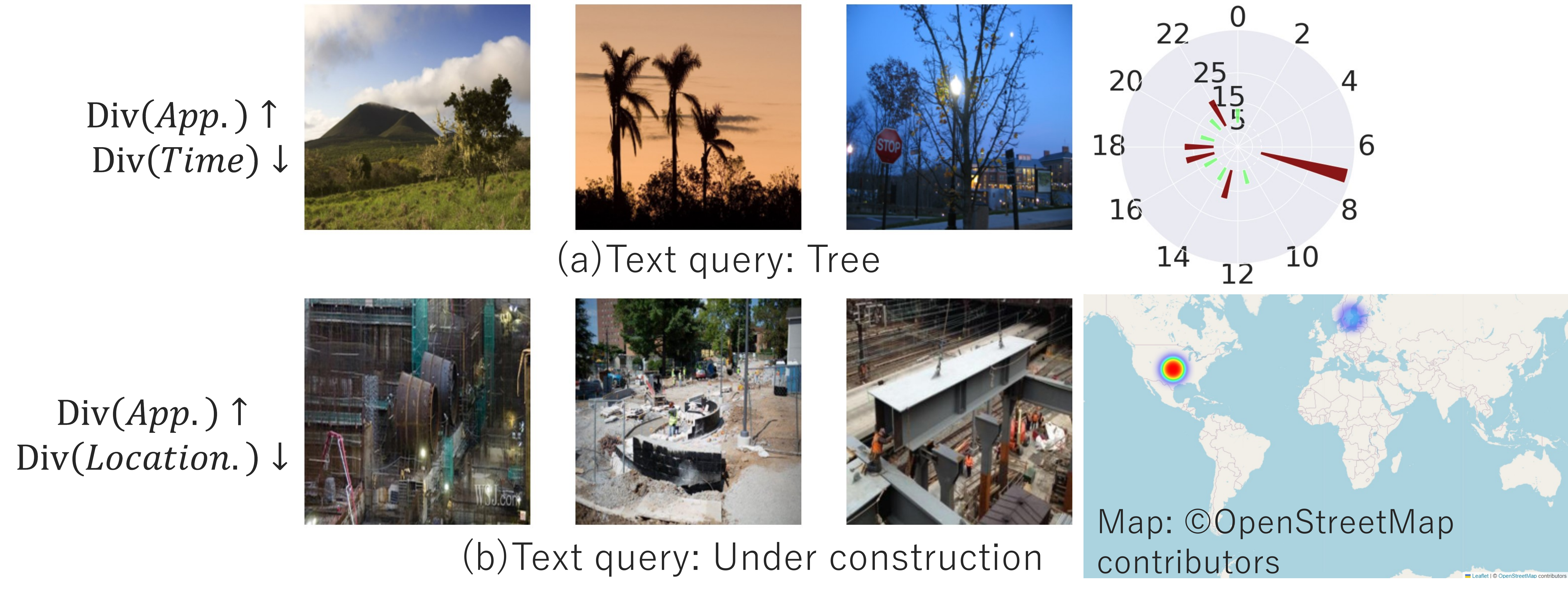}
        \caption{Examples of the results by MS-DPPs.
        The upper figures show shooting time refinement results, including the top three images and polar histograms of shooting times in the top 20 images.
        The upward bar of the histogram is at 0:00, and time moves clockwise, rounding the clock in 24 hours.
        The right figures show location refinement results, including heatmaps of the locations in the top 20 images.
        }
        \label{fig:exp_example}
    \end{figure}

    \subsubsection{Results for Diversity Decreasing Tasks}
    Table~\ref{tab:exp_main_dec} displays the results of our method and baselines for tasks focused on increasing appearance diversity while reducing diversity in a specific attribute. Consistent with previous findings, our method demonstrates superior performance, confirming its efficacy. Notably, there is a significant improvement in diversity scores, supporting the effectiveness of our attribute integration approach in these tasks.
    
    Figure~\ref{fig:exp_example} shows the results of two CDR-CA tasks by MS-DPPs.
    The results demonstrate that our method can effectively refine the diversity of multiple attributes while maintaining highly relevant images~\footnote{The supplementary material includes additional examples.}.

    \subsubsection{Effect of Tangent Normalization}
    
\iftrue
\begin{table}[t]
    \begin{center}
        \begin{subtable}[t]{0.45\textwidth}
            \centering
            \scalebox{0.75}{
                \begin{tabular}{l|ll|rr}
\toprule
                &TVs&$\M$& {\bf VG (Time)} & {\bf I1M (Location)} \\
\midrule
                &-&- & \first{0.8649} \small{(0.8220, 0.9124)} & \first{0.8081} \small{(0.7618, 0.8604)} \\
                MS-DPP&\checkmark&- & \second{0.8557} \small{(0.8211, 0.8934)} & 0.7935 \small{(0.7649, 0.8244)} \\
                &\checkmark&\checkmark & 0.8425 \small{(0.7893, 0.9034)} & \second{0.7938} \small{(0.7651, 0.8247)} \\
\bottomrule
                \end{tabular}
            }
            \caption{}
            \label{tab:exp2_tn_inc}
        \end{subtable}%
        \hfill
        \begin{subtable}[t]{0.45\textwidth}
            \centering
            \scalebox{0.75}{
                \begin{tabular}{l|ll|rr}
\toprule
                &TVs&$\M$& {\bf VG (Time)} & {\bf I1M (Location)} \\
\midrule
                &-&- & \first{0.5746} \small{(0.6616, 0.5079)} & \fourth{0.6351} \small{(0.7577, 0.5466)}\\
                MS-DPP&\checkmark&-& \second{0.5574} \small{(0.6687, 0.4778)} & \second{0.6373} \small{(0.7650, 0.5462)} \\
                &\checkmark&\checkmark & \fourth{0.5558} \small{(0.6678, 0.4759)} & \first{0.6393} \small{(0.7598, 0.5518)} \\
\bottomrule
                \end{tabular}
            }
            \caption{}
            \label{tab:exp2_tn_dec}
        \end{subtable}
        \caption{Ablation study on TN for the tasks of (a) increasing diversities of all attributes and (b) increasing diversity in appearance while decreasing diversity in time or location.
        The second and third columns indicate whether TN is applied for tangent vectors of similarity matrices $\Sim_i$ and the unified similarity matrix $\M$, respectively.}
        \label{tab:exp2_combined}
    \end{center}
\end{table}
\fi

    Tables~\ref{tab:exp2_combined} and \ref{tab:exp_controll} show the results of the proposed method with and without TN.
     Although the overall performance is not always improved, PRS is significantly improved by TN, as shown in Table~\ref{tab:exp_controll}.
     High PRS is crucial for practical applications, enabling users to adjust each attribute's importance interactively.
     This result supports the effectiveness of TN in improving the reflection of user preferences, which is the primary motivation for introducing TN.

    \begin{table}[tb]
        \centering
        \scalebox{0.75}{
        \begin{tabular}{lrrrr}
            \toprule
             & \multicolumn{2}{c}{{\textbf{VG (Time)}}} & \multicolumn{2}{c}{{\textbf{I1M (Location)}}} \\
             & Inc. & Dic. & Inc. & Dic. \\
            \midrule
            Clustering & 1.0876 & 0.8445 & 0.7180 & 0.4514 \\
            MMR & -0.3615 & 0.0873 & -0.2449 & 0.1192 \\
            k-DPP & 1.3000 & -0.1115 & 0.0171 & -0.0561 \\
            \midrule
            MS-DPP & -0.2842 & 0.7341 & -0.7272 & 0.4821 \\            \hdashline
            ~+ TN (TVs) & \second{1.3586} & \second{1.3358} & \second{1.2849} & \second{1.1936} \\
            ~+ TN (TVs + M) & \first{1.4073} & \first{1.4196} & \first{1.2955} & \first{1.2642} \\
            \bottomrule
        \end{tabular}
        }
        \caption{PRS~\eqref{eq:prs} on diversity increase (Inc.) and decreasing (Dec.) tasks.}
        \label{tab:exp_controll}
    \end{table}

    \section{Conclusion}
    We proposed a novel task, Contextual Diversity Refinement of Composite Attributes (CDR-CA), which is a generalization of the standard result diversification task.
    CDR-CA targets multiple attributes of each image and has mixed objectives, i.e., increasing diversity in some attributes while decreasing diversity in others.
    CDR-CA opens up the possibility of addressing various practical applications, unlike the standard result diversification task.
    To address CDR-CA, we proposed a tailored method, Multi-Source Determinantal Point Processes (MS-DPPs).
    The basis of MS-DPP is to consider the product of multiple DPP models.
    We unified the basic formulation into a single DPP model through operations on the SPD manifold.
    This unification enables to employ the established optimization techniques from the standard DPPs and motivates the introduction of Tangent Normalization (TN), a novel technique that aligns attribute diversities with user-defined weights.
    Our experiments demonstrated the effectiveness of MS-DPPs in enhancing the diversity of multiple attributes while preserving highly relevant images.

\newpage
\renewcommand{\thetable}{\Alph{table}} 
\renewcommand\thesection{\Alph{section}}
\renewcommand\thefigure{\Alph{figure}}

\setcounter{section}{0}
\setcounter{table}{0}
\setcounter{figure}{0}

{\noindent\LARGE{{\textbf Supplementary Material}}}

    \section{Preliminary: SPD Manifold}
    \label{sec:SPD}
    Our MS-DPPs are expressed with operations on the Symmetric Positive Definite (SPD) manifold.
    Thus, we introduce the concept of an SPD manifold and two important operations: the matrix logarithm and exponential maps~\cite{pennec2006riemannian}.
    A set of SPD matrices forms a Riemannian manifold called the SPD manifold. 
    As an SPD manifold is not an Euclidean space, we need to use a Riemannian distance to measure the distance between two SPD matrices.
    
    The alternative approach is to use a tangent vector space, an Euclidean space that approximates the SPD manifold at a point.
    Matrix logarithm $\logm$ maps an SPD matrix to a tangent vector with the identity matrix as the origin.
    This map enables us to use Euclidean operations on the tangent vector space, such as distance measurement, addition, and scalar multiplication.
    Matrix exponential $\expm$ is the inverse map of $\logm$, i.e., it maps a tangent vector to an SPD matrix.
    The matrix exponential and logarithm of an SPD matrix are defined as follows:
    \begin{align}
        \expm\Sim=\Phi \diag{e^{\lambda_1},e^{\lambda_2},...,e^{\lambda_d}}\Phi^T\\
        \logm\Sim=\Phi \diag{\log{\lambda_1},\log{\lambda_2},...,\log{\lambda_d}}\Phi^T,
    \end{align}
    where $\Sim\in{\mathbb R}^{d\times d}$ is an SPD matrix, $\Phi$ is the matrix lining eigenvectors up as columns, $\seto{\lambda_i}$ are the eigenvalues of $\Sim$.

    \section{Unified Representation for MS-DPP}
    \begin{theorem}
        {\rm\textbf{Unified Representation:}} The composite DPP model ($f_{base}$) in (4) of the main paper can be expressed as a single DPP model with a unified similarity matrix $\M$ as follows:
    \begin{align}
        f_{base}(\Y{g}) 
        \propto\mdet{\Rmat_{\Y{g}}}\mdet{\M}\mdet{\Rmat_{\Y{g}}},\label{eq:msddp1}
    \end{align}
    where $\M=\expm{\sum_{i=1}^{N_A}{\logm \Sim_{i,\Y{g}}}}$.
    \end{theorem}

    \begin{proof}
    Let $\Amat_{i}= \logm \Sim_{i,\Y{g}}$. It is known that $\mdet{\Sim_i}= e^{\tr{\Amat_i}}$.
    Then, we rewrite $f_{base}$ with the natural logarithm as follows:
    \begin{align}
        \label{eq:msdpps-log}
        \log{f_{base}(\Y{g})}
        =a+\sum_{i=1}^{N_A}{\log\mdet{\Sim_{i,\Y{g}}}}
        \nonumber\\
        =a+\sum_{i=1}^{N_A}{\log e^{\tr{A_i}}}
        =a+\sum_{i=1}^{N_A}\tr{{\Amat_i}},
    \end{align}
    where $a={\log\mdet{\Rmat_{\Y{g}}}}^{2N_A}$. 
    From \eqref{eq:msdpps-log}, we can obtain the following equation:
    \begin{align}
        f_{base}(\Y{g})
        =e^{N_A}\mdet{\Rmat_{\Y{g}}}(e^{\tr \sum_{i=1}^{N_A}{\Amat_{i,\Y{g}}}})\mdet{\Rmat_{\Y{g}}}
        \nonumber\\\propto\mdet{\Rmat_{\Y{g}}}\mdet{\expm{\sum_{i=1}^{N_A}{\Amat_{i,\Y{g}}}}}\mdet{\Rmat_{\Y{g}}} 
        \nonumber\\=\mdet{\Rmat_{\Y{g}}}\mdet{\expm{\sum_{i=1}^{N_A}{\logm \Sim_{i,\Y{g}}}}}\mdet{\Rmat_{\Y{g}}}
        =\mdet{\Rmat_{\Y{g}}}\mdet{\M}\mdet{\Rmat_{\Y{g}}}.
    \end{align}
    \end{proof}

    \section{Experimental Settings and Additional Results}
    \subsection{Datasets} We use three datasets, i.e., Visual Genome (VG)~\cite{vg}, Incidents 1M (I1M)~\cite{weber2020eccv,I1M}, and PixelProse (PP)~\cite{pp}, as they have enough amount of images and have necessary attributes.
    We employ VG for tasks involving shooting time, I1M for tasks involving shooting location, and PP for tasks involving shooting time or location.
    We selected images with the necessary attributes from each dataset.
    Attributes were obtained from EXIF information of each image in VG and PixelProse. In I1M, shooting locations were obtained following the methodology described in the original paper~\cite{I1M}.
    In I1M, we added a small amount of uniform noise to locations that are mapped to exactly the same location.
    Resulting the numbers of images are 649, 22,547 and 25,151 for VG, I1M, and PP, respectively.

    For text queries of each dataset, we use fifteen dominant object names in VG, the disaster types and shooting locations types originally annotated in I1M, and the 1,000 randomly sampled dense captions in PP.
    The text queries of VG are ``person'', ``building'', ``table'', ``vehicle'', ``animal'', ``tree'', ``sky'', ``road'', ``chair'', ``light'', ``car'', ``desk'', ``bird'', ``apple'', and ``dog''.

    We use 20\% of images in each dataset as validation set, with the remaining images constituting the test set.
    The grid search algorithm utilized the validation set for hyperparameter tuning.

    \subsection{Baselines} 
    We benchmark our method against the original results from BLIP-2, and three single-source diversification techniques: Clustering-based method, MMR, and k-DPP.
    For the diversity increasing tasks, we employ the standard MMR and k-DPP algorithms with average image appearance features or embedded attribute features.
    For the diversity decreasing tasks, we slightly modify MMR and k-DPP; we multiply the similarity matrix of each attribute by -1.

    The Clustering-based method is implemented as follows: k-means clustering is first applied to the input features. Clusters are ranked based on their average value of retrieval scores in each cluster, and images are sequentially selected from clusters according to this ranking and relevance scores for diversity-increasing tasks. For diversity-decreasing tasks, all images in the first-ranked cluster are selected, followed by the second-ranked cluster, and so on.

    As multi-source baselines, we input concatenated features of appearance and attribute features to the Clustering-based method and MMR.
    Attribute weights $w_i$ and task indicator $s_i$ are multiplied to feature vectors of each attribute, before the concatenation.
    As a multi-source DPP baseline, we input the average weighted similarity matrix of each attribute to k-DPP.
    Attribute weights and task indicators are also multiplied by the similarity matrix of each attribute before the averaging.

    The hyperparameters of the baselines were tuned by the grid-search algorithm on the validation set in terms of the HM.
    The trade-off parameter $\theta$ between retrieval accuracy and diversity was selected from 0.01 to 0.9 in increments of 0.1.
    The attribute weights $\seto{w_i}$ were selected from 0.1 to 0.9 in increments of 0.2 and normalized to sum to 1.
    The number of clusters in the Clustering-based method was selected from 40, 60, and 80.

    \begin{figure}[tb]
            \centering
            \includegraphics[width=\columnwidth]{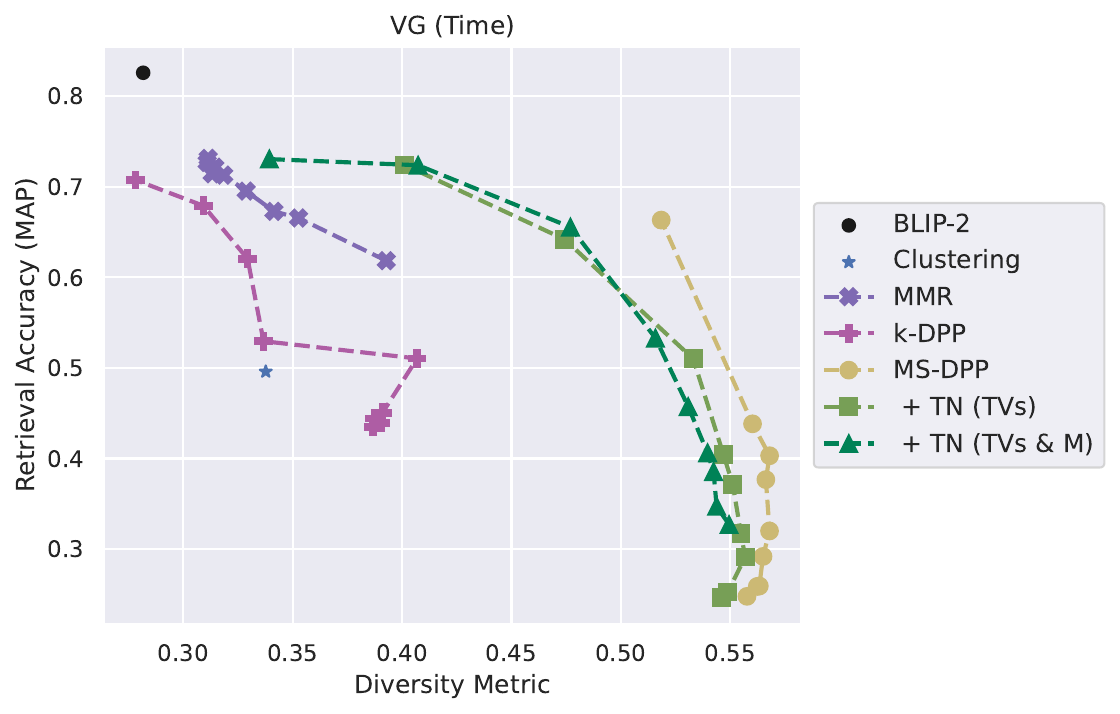}
            \caption{Relationship between retrieval accuracy and diversity measurements against the trade-off parameter $\theta$. The task is to increase diversity in appearance while decreasing diversity in shooting time. We set attribute weights $w_i=0.5$ for all attributes and vary $\theta$ from 0.1 to 0.9 with a step size of 0.1.}
                \label{fig:exp_alpha}
    \end{figure}

    \subsection{Additional Experimental Results}
    \subsubsection{Effect of Trade-off Parameter}
    MS-DPPs have the trade-off parameter $\theta$, which balances retrieval accuracy and diversity metric.
    Figure~\ref{fig:exp_alpha} depicts retrieval accuracies and diversity metrics as functions of $\theta$.
    As anticipated, $\theta$ influences the balance: higher $\theta$ enhances retrieval accuracy, whereas lower $\theta$ boosts diversity measurements.
    This demonstrates that the trade-off parameter effectively balances retrieval accuracy and diversity adjustment.
    Additionally, TN expands the range of achievable accuracies and diversities, further confirming its efficacy in enhancing this trade-off.
    This improvement is attributed to normalization, which mitigates the influence of tangent vector norms by aligning them with the relevance matrix norm.

    \subsubsection{Effect of the Number of Attributes}
    Table \ref{tab:three_sources} shows the results of the three attributes mixed-direction task on PixelProse and the single-attribute (shooting time) diversity-decreasing task on Visual Genome.
    The proposed method outperforms the baselines in both tasks, suggesting that the proposed method is effective in diverse CDR-CA settings.

    \begin{table}[t]
    \begin{center}
        \scalebox{0.75}{
            \begin{tabular}{lr|r}
\toprule
                & {\bf PP} & {\bf VG} \\
\midrule
                & HM$\uparrow$ \footnotesize{(N@10$\uparrow$,DM$\uparrow$)}& HM$\uparrow$ \footnotesize{(MAP$\uparrow$,DM$\uparrow$)} \\
\midrule
            BLIP-2 & 0.2680 \footnotesize{(0.6214, 0.1708)}  & 0.3056 \footnotesize{(0.8259, 0.1875)} \\
            k-DPP & 0.3067 \footnotesize{(0.4761, 0.2262)} & 0.2938 \footnotesize{(0.7359, 0.1836)}\\
\midrule
            MS-DPP & \textbf{0.3535} \footnotesize{(0.4580, 0.2878)} &  \textbf{0.4570} \footnotesize{(0.7069, 0.3377)}\\
\bottomrule
        \end{tabular}
    }
    \caption{Results on (left) the mixed task of increasing diversity in appearance while decreasing diversity in shooting time and location, and (right) the diversity decreasing task in shooting time.}
    \label{tab:three_sources}
    \end{center}
    \begin{center}
        \scalebox{0.75}{
            \begin{tabular}{lrrrrr}
\toprule
                N& 200 & 400 & 600 & 800 & 1,000 \\
\midrule
                k-DPP & 0.47 & 1.92 & 4.14 & 7.36 & 11.60 \\
				MS-DPP & 0.62 & 2.22 & 4.72 & 8.24 & 12.73 \\
\bottomrule
        \end{tabular}
   }
    \caption{
    Computational time [seconds] per one retrieval for different gallery sizes (N).
    The experiment is conducted on GTX1080Ti, 128GB RAM, and E5-2620 v4 with PyTorch 2.4.1 CUDA 11.8.
    }
    \label{tab:runtime}
    \end{center}
\end{table}

    \subsubsection{Runtime}
    Table \ref{tab:runtime} shows computation times for varying gallery sizes.
    Matrix log/exp operations increase runtime but are not dominant.
    The common practice in result diversification is to re-rank only the top-K images to mitigate the runtime issue.
    For faster approximations, we can use efficient log/exp manifold operations defined by other SPD manifold metrics, e.g. Log-Euclidean and Bures Wasserstein metrics.

    \subsubsection{Additional Examples}
    Figures~\ref{fig:supp_time_example} and~\ref{fig:supp_loc_example} show additional examples of the results of MS-DPPs on four CDR-CA tasks.
    MS-DPPs effectively refine the diversity of multiple attributes, faithfully reflecting the application contexts while maintaining highly relevant images.

    \begin{figure*}[tb]
        \centering
        \includegraphics[width=\textwidth]{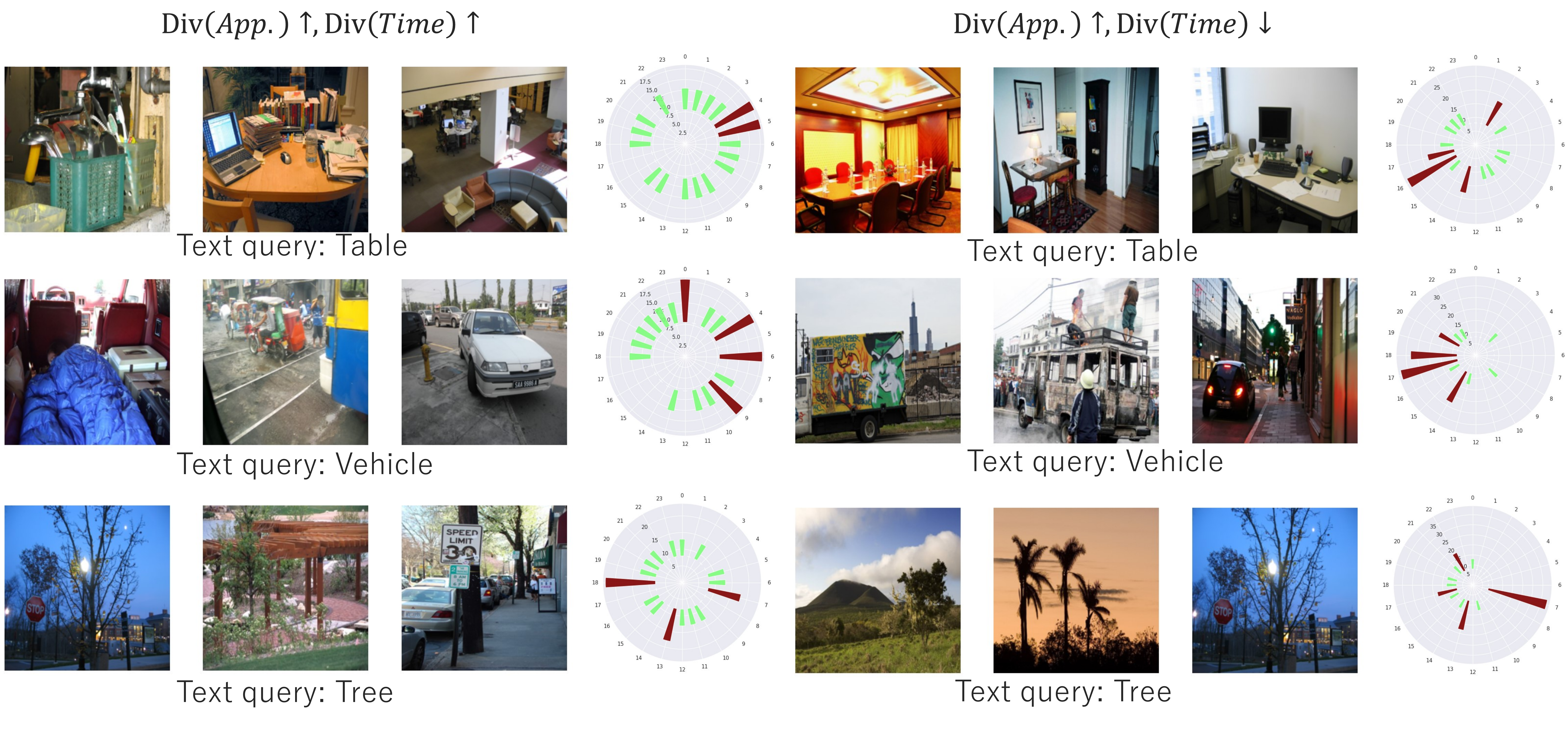}
        \caption{Additional examples of the results of MS-DPPs for three text queries. Each result includes the top three images and a polar histogram of the shooting times in the top 20 results. The upward bar of the histogram is at 0:00, and time moves clockwise, rounding the clock in 24 hours.}
        \label{fig:supp_time_example}
    \end{figure*}

    \begin{figure*}[tb]
        \centering
        \includegraphics[width=\textwidth]{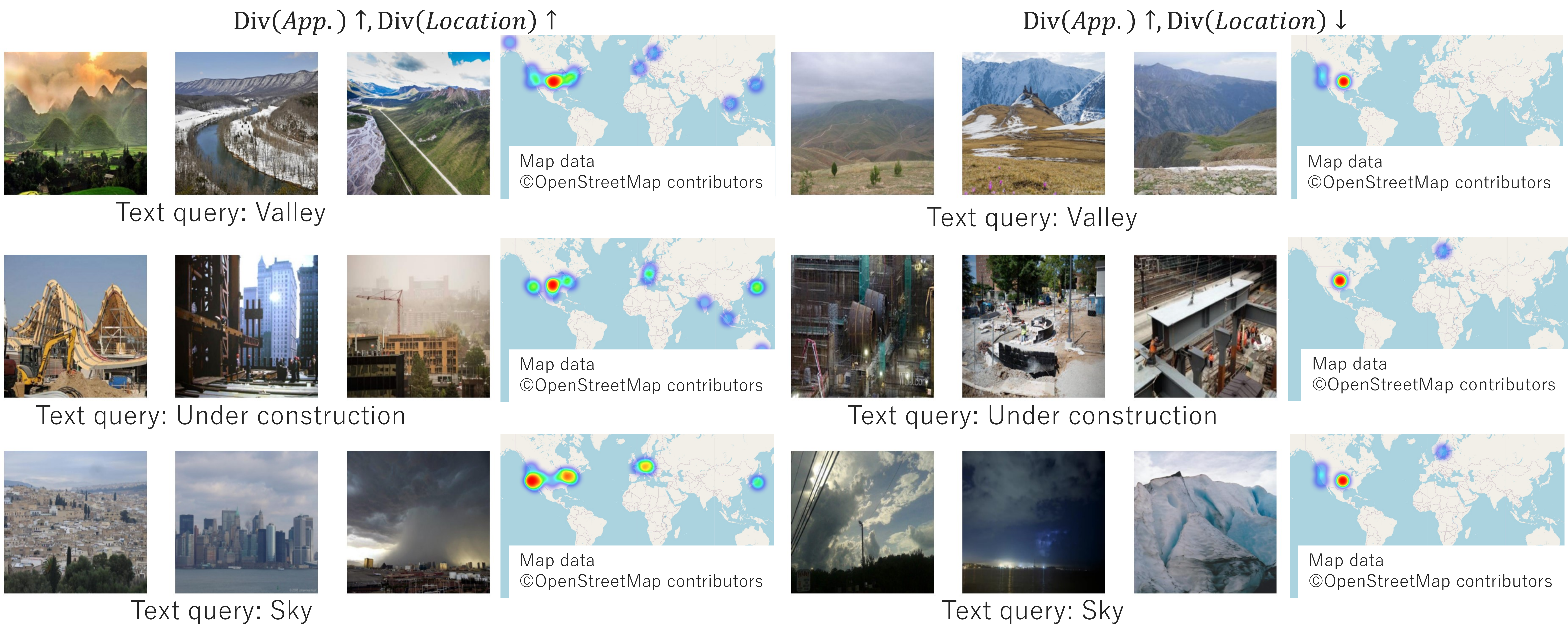}
        \caption{Additional examples of the results of MS-DPPs. Each result includes the top three images and a heatmap of the shooting locations in the top 20 results.}
        \label{fig:supp_loc_example}
    \end{figure*}

\clearpage
\bibliographystyle{named}
\bibliography{refs}

\end{document}